 \documentclass{colt2012_1} 

\usepackage{amsmath, amssymb, multirow, paralist}
\usepackage{algorithm,algorithmic}
\usepackage{color}

\usepackage{hyperref}
\hypersetup{
colorlinks=true,
citecolor=blue,
linkcolor = {red},
linkbordercolor={1 1 1}, 
citebordercolor={1 1 1} 
}

\firstpageno{1}

\def \Kh {\widehat {K}}

\def \x {\mathbf{x}}
\def \H {\mathcal{H}_{\kappa}}
\def \R {\mathbb{R}}
\def \w {\mathbf{w}}

\def \a {\mathbf{a}}

\def \E {\mathrm{E}}
\def \C {\mathcal{C}}

\def \y {\mathbf{y}}
\def \Hki {\mathcal{H}_{\kappa_i}}
\def \Hkj {\mathcal{H}_{\kappa_j}}

\def \fh {\widehat{\f}}

\def \S {\mathcal{S}}

\def \L {\mathcal{L}}

\def \span {\mbox{span}}

\def \D {\mathcal{D}}
\def \E {\mathcal{E}}
\def \X {\mathcal{X}}
\def \H {\mathcal{H}}
\def \fh {\hat{f}}

\def \K {\mathcal{K}}

\def \h {\mathbf{h}}
\def \ld {\ell_2({\D})}

\begin{document}

\title[Sparse MKL with Geometric Convergence]{Sparse Multiple Kernel Learning with \\ Geometric Convergence Rate}
 \coltauthor{\Name{Rong Jin} \Email{rongjin@cse.msu.edu}\\
 \addr Department of Computer Science and Engineering\\
       Michigan State University\\
       East Lansing, MI, 48824, USA\\
 \Name{Tianbao Yang} \Email{tyang@ge.com}\\
 \addr Machine Learning Lab \\
 GE Global Research\\
 San Ramon, CA 94583,, USA\\
\Name{Mehrdad Mahdavi} \Email{mahdavim@cse.msu.edu}\\
       \addr Department of Computer Science and Engineering\\
       Michigan State University\\
       East Lansing, MI, 48824, USA
 }


\maketitle

\begin{abstract}
In this paper, we study the problem of  sparse multiple kernel learning (MKL), where the goal is to efficiently learn a combination of a fixed small number of kernels from a large pool that could lead to a kernel classifier with a small prediction error. We develop an efficient algorithm based on the greedy coordinate descent algorithm, that is able to achieve a geometric convergence rate under appropriate conditions. The convergence rate is achieved by measuring the size of functional gradients by an empirical $\ell_2$ norm that depends on the empirical data distribution. This is in contrast to previous algorithms that use a functional norm to measure the size of gradients, which is independent from the data samples. We also establish a  generalization error bound of the learned  sparse kernel classifier using the technique of local Rademacher complexity.

\end{abstract}
\begin{keywords}
kernel methods, multiple kernel learning, greedy coordinate descent, generalization bound
\end{keywords}

\section{Introduction}
Kernel methods have been studied extensively, thanks to their empirical success in a variety of applications. Examples of kernel methods include support vector machines (SVMs), kernel ridge regression, kernel clustering, kernel PCA, and  many others. It is well known that the choice of kernel function can be crucial to the success of kernel methods. Although, in principle kernel can be chosen by standard model selection methods such as cross validation, the high computational cost makes it  unattractive. Over the past decade, significant progress has been made to efficiently learn an appropriate kernel for a given task.


Among the many approaches developed for kernel learning, recent studies have been  focused predominately on multiple kernel learning (MKL) algorithms. Given a collection of kernels, the objective of MKL is to learn a combination of multiple kernel classifiers, one for each kernel function, from the training examples that results in small prediction error. Many computational algorithms have been developed for multiple kernel learning~\citep{Lanckriet:2004:LKM,argyriou:2005:learning,Bach:2008:mkl,Argyriou:2006:DC,Lewis:2006:NKC,Micchelli:2005:mkl,Ong:2005:mkl,Bach:2004:MKL,Rakotomamonjy:2008:mkl,Sonnenburg:2006:LSM,xu:2008:level, citeulike:9417296}. The analysis of generalization error bound for MKL has been developed in several studies~\citep{hussain:2011:note,Ying:2007:mkl,Cortes:2009:LRL:1795114.1795128,Cortes:2010:mkl,Bousquet:2003:mkl,Srebro:2006:mkl,Ying:2009:mkl}, aiming to bound the additional error arising from optimizing the combination of multiple kernels. These studies have shown that MKL can be effective even when the number of kernels to be combined is very large. For instance, the generalization error bound from learning a combination of $m$ different kernels, will only deteriorate by a factor of $\log m$ when the sum of kernel combination weights is bounded.

Despite the encouraging results, one problem with MKL is that the resulting classifier can be a combination of many kernel classifiers, leading to a high computational cost in testing.  We address this challenge by developing efficient algorithms and theories for  sparse multiple kernel learning. The objective of  sparse MKL is to learn a sparse combination of multiple kernel classifiers involving no more than $d$ kernels, where $d \ll m$ is a predefined constant.

We develop a simple algorithm for learning such a sparse combination of multiple kernel classifiers, and present the analysis  bounding the generalization performance of the learned kernel classifier. Our algorithm is an iterative algorithm  based on the  greedy coordinate descent algorithm~\citep{shai-2010-trade,nesterov-2007,tseng-2011-bcd}.  To generate a sparse MKL solution involving no more than $d$ kernels, at each iteration, our algorithm adds to the existing pool the kernel with the largest gradient. The size of gradients is measured by an empirical $\ell_2$ norm that depends on the training examples.  Under appropriate condition, the proposed approach is able to achieve a geometric convergence rate. 
To the best of our knowledge, this is the first algorithm for sparse MKL that achieves a geometric convergence rate.

Although several algorithms have been developed for sparse MKL by exploring different forms of regularization~\citep{Vishwanathan-2010-mkl,kloft-2009-efficient,ICML2011:Orabona:mkl}, none of them are able to establish the generalization error bound for a MKL solution involved a fixed number (i.e., $d$) of kernels. We also note that our work differs from the studies on the sparsity of MKL ~\citep{Koltchinskii:2008:sparsity,Koltchinskii:2010:mkl} which focus on bounding the sparsity of combination weights for kernels and do not address our problem directly. 

The most related work to this study is~\citep{omp-mkl-nips-2011}, where a group orthogonal matching pursuit (GOMP) algorithm is applied to learn a sparse combination of kernel classifiers with exactly  $d$ kernels. Unlike previous formulations for sparse MKL that use $\ell_1$ regularization (i.e. $\sum_{j}\|f_j\|_{\Hkj}$), the authors propose to use $\ell^2_2$ regularization (i.e. $\sum_{j}\|f_j\|^2_{\Hkj}$) together with a sparsity constraint (i.e. $\ell_0$ constraint) for sparse MKL.  Although they did not present a convergence analysis for the proposed algorithm except for a sparse recovery analysis, we can apply the analysis in~\citep{shai-2010-trade} for smooth functions to their algorithm to obtain a $O(1/d)$ convergence rate. The group orthogonal matching pursuit algorithm is similar to the greedy coordinate descent algorithm used in this study except that we measure the size of gradients by an empirical $\ell_2$ norm while it is measured by a functional norm in ~\citep{omp-mkl-nips-2011}. It is this difference that leads to a geometric convergence rate for the proposed algorithm which is a significant improvement over the rate of $O(1/d)$.

\textbf{Outline of  contributions.} The following contributions are made in this paper: \begin{itemize}
\item We present a baseline algorithm, based on the greedy coordinate descent method, that achieves $O(1/d)$ convergence rate when using $\ell_1$ norm functional regularizer.
\item We introduce an empirical $\ell_2$ norm to measure the size of functional gradients in the application of greedy coordinate descent algorithm to sparse MKL, and achieve a geometric convergence rate under appropriate conditions.
\item We study the generalization performance of the proposed algorithm. Specifically, we derive an upper bound on the generalization performance of learned classifier using local Rademacher technique  that has a additive term of $O\left(d\sqrt{\ln m/N}\right)$, which matches the existing bounds in their dependence on $m$ (i.e., the number of kernel functions) and $N$ (i.e., the number of training samples).
\end{itemize}

Our paper is organized as follows. In the next section we formally introduce the problem of  sparse MKL. In section  \ref{sec:warmup} we present our baseline algorithm with its convergence analysis. Section \ref{sec:alg} introduces the main algorithm proposed in this paper with analysis of its convergence rate and generalization bound. We wrap up in Section \ref{sec:conc} with a discussion of possible directions for the future work.

\section{Problem Setting: Sparse Multiple Kernel Learning (MKL)}

Let $\D = \{(\x_i, y_i), i=1, \ldots, N\}$ be a collection of training examples, where $\x_i \in \X$ and $y_i \in \{-1, +1\}$, and let $\{\kappa_j(\cdot, \cdot): \X\times\X \mapsto \R, j \in [m]\}$ be a collection of reproducing kernels to be combined, where $[m]$ denotes the set $\{1,\cdots, m\}$. Let $\{\H_j, j\in[m]\}$ be the associated  Reproducing Kernel Hilbert Spaces (RKHS). We denote by $\y = (y_1, \ldots, y_N)^{\top}$ the outputs for all the instances in $\D$. For the convenience of analysis, we assume $\kappa_j(\x, \x) \leq 1$ for any $\x \in \X$ and any $j \in [m]$. The goal of MKL is to learn a function $f = \sum_{j=1}^m f_j$, where $f_j \in \H_j, j \in [m]$, that has a small generalization error. A common approach for MKL is to learn the combination of kernel classifiers by solving the following optimization problem~\citep{Micchelli:2005:mkl}
\begin{eqnarray}
    \min\limits_{f \in \H} \quad \L(f) = \frac{1}{N}\sum_{i=1}^N \ell(f(\x_i), y_i) + \lambda \sum_{j=1}^m \|f_j\|_{\H_j},\label{eqn:1}
\end{eqnarray}
where $\H = \{f = \sum_{j=1}^m f_j: f_j \in \H_j\}$, and $\ell(z, y) = (z-y)^2/2$ is a square loss~\footnote{Although we restrict our discussion to square loss, it is straightforward to extend our result to the quadratic-type loss function defined in~\citep{Koltchinskii:2010:mkl}}. In this study, we assume that the number of kernels $m$ is very large (could be larger than the number of training examples $N$), and our objective is to learn a combination of kernel classifiers involving no more than $d$ kernels, where $d \ll m$ is a predefined constant. For the convenience of discussion, we define by $\E_N(f) = \frac{1}{N}\sum_{i=1}^N \ell(f(\x_i), y_i)$ the empirical loss for kernel classifier $f$, by $\|f\| = \sum_{j=1}^m\|f_j\|_{\H_j}$ the norm of a combined kernel classifier $f$, and by $J(f) = \{j \in [m]: f_j \neq 0\}$ the subset of non-zero kernel classifiers used to construct $f$. Finally, we define $f^*$ the optimal solution to (\ref{eqn:1}), i.e.,
\begin{eqnarray}
    f^* = \mathop{\arg\min}\limits_{f \in \H} \L(f).
\end{eqnarray}

Note that according to~\citep{Micchelli:2005:mkl}, the problem in (\ref{eqn:1}) is equivalent to the following optimization problem
\begin{eqnarray}
    \min\limits_{\gamma \in \R_+^m, \gamma^{\top}\mathbf{1} \leq 1} \min\limits_{f \in \H_{\gamma}} \frac{\lambda'}{2}
    \|f\|_{\H_{\gamma}}^2 + \frac{1}{N}\sum_{i=1}^N (f(\x_i) - y_i)^2, \label{eqn:mkl}
\end{eqnarray}
where $\lambda' > 0$ is an appropriately chosen parameter depending on $\lambda$ in (\ref{eqn:1}), and $\H_{\gamma}$ is a RKHS endowed with a combined kernel function $\kappa(\cdot, \cdot;\gamma) = \sum_{i=1}^m \gamma_i \kappa_j(\cdot, \cdot)$. It is not difficult to show that $\gamma_j$ computed in (\ref{eqn:mkl}) is proportional to $\|f_j\|_{\H_j}$ computed from (\ref{eqn:1}). As a result, choosing the kernel classifiers with the largest functional norm in (\ref{eqn:1}) is equivalent to choosing the kernels with the largest weights $\gamma_j$ in (\ref{eqn:mkl}).


\section{Warmup: A Greedy Coordinate Descent Algorithm  for  Sparse MKL}
\label{sec:warmup}

A straightforward approach for  sparse MKL is a two-stage scheme: it first learns a combination of all $m$ kernels by solving the problem in (\ref{eqn:1}) and then only keeps the $d$ most ``important'' kernel classifiers $f_j$ in the combination. To select the most important kernel classifiers, a simple approach is to choose the kernel classifiers with the largest functional norm $\|f_i\|_{\Hki}$,  because $\|f_i\|_{\Hki}$ is proportional to the combination weight $\gamma_i$ in~(\ref{eqn:mkl}). It is however easy to construct a counter example to show that the two-stage scheme fails to find the best kernel. In particular, we will show that for two cases that have the same sets of unique kernels, the two-stage scheme chooses different kernels. In the first case, we have two kernel functions $\kappa_1(\cdot, \cdot)$ and $\kappa_2(\cdot, \cdot)$. Using multiple kernel learning, we can learn the weights for both kernels. Let the learned weights be $0.8$ for $\kappa_1(\cdot, \cdot)$ and $0.2$ for $\kappa_2(\cdot, \cdot)$. According to the two-stage approach, we will select kernel $\kappa_1(\cdot, \cdot)$. In the second case, we have $10$ identical copies of $\kappa_1(\cdot, \cdot)$ and one copy of $\kappa_2$. Since both cases share the same set of unique kernels, we expect the same kernel to be selected by the two-stage approach. However, based on the symmetric argument, it is straightforward to show that the weight for $\kappa_2(\cdot, \cdot)$ remains unchanged while the weights for the copies of $\kappa_1(\cdot, \cdot)$ are reduced to $0.08$. As a result, the two-stage approach selects kernel $\kappa_2(\cdot, \cdot)$ for the second case, a different kernel from the first case. Another problem with this two-stage approach is its high computational complexity since it requires solving an optimization problem involved all kernel functions, even including the ones that are totally irrelevant to the target prediction task.

As the first step, we present a baseline algorithm that extends the greedy coordinate descent algorithm~\citep{shai-2010-trade} to solve the $\ell_1$ regularized MKL in~(\ref{eqn:1}) and achieves a $O(1/d)$ convergence rate. The basic steps are shown in Algorithm~\ref{alg:1}. At each iteration $k$, Algorithm~\ref{alg:1} selects the kernel with the largest gradient measured by its functional norm, denoted by $j_k$, and expands the set of selected kernels $\S_k$ to $\S_{k+1}$ by including $j_k$. It then searches for the optimal combination of kernels in the set $\S_{k+1}$ that minimizes the objective function $\L(f)$.  Note that although the objective in~(\ref{eqn:1}) is non-smooth due to the non-smooth regularization term $\sum_{j=1}^m \|f_j\|_{\H_j}$, we are still able to  obtain a $O(1/d)$ convergence rate as shown in Theorem~\ref{lemma:1}. The magic lies in step 4, where instead of choosing the coordinate with the largest gradient with respect to the objective function $\L(f)$, we choose the coordinate with the largest gradient with respect to $\E_N(f)$, the smooth part in the objective function, i.e.
\[
    \left\|\nabla_j \E_N(f)\right\|_{\Hkj} = \left\|\frac{1}{N}\sum_{i=1}^N \ell'(f(\x_i), y_i) \kappa_j(\x_i, \cdot)\right\|_{\Hkj}.
\]
On the other hand, in step 7, we update the multiple kernel classifier by solving the $\ell_1$ regularized MKL. It is this special design that makes it possible to achieve $O(1/d)$ convergence rate even for the non-smooth objective function in (\ref{eqn:1}). We finally note that Algorithm~\ref{alg:1} is similar in spirit to the GOMP based approach~\citep{omp-mkl-nips-2011} and share the same convergence rate. The main difference is that we directly solve the $\ell_1$ regularized MKL in (\ref{eqn:1}) while in~\citep{omp-mkl-nips-2011}, a $\ell_2^2$ regularizer is used and the sparsity is enforced through a constraint based on the $\ell_0$ norm.

\begin{algorithm}[t]
\center \caption{A Greedy Coordinate Descent Approach for Sparse MKL with $\ell_1$ Regularization}
\begin{algorithmic}[1] \label{alg:1}
    \STATE {\bf Input}: $\lambda > 0$: regularization parameter,   $d$: the number of selected kernels

    \STATE {\bf Initialization}: $f_j^{0} = 0, j \in [m]$ and $\S_{0} = \emptyset$.
    \FOR{$k = 1, \ldots, d$}
        \STATE $j_k = \mathop{\arg\max}_{j \in [m]} \left\|\nabla_j \E_N(f^{k-1})\right\|_{\H_j}$
        \STATE Exist the loop if $\left\|\nabla_{j_k} \E_N(f^{k-1})\right\|_{\H_{j_k}} \leq \lambda$
        \STATE $\S_{k} = \S_{k-1} \cup \{j_k\}$
        \STATE Update the kernel classifier by solving the following optimization problem
        \begin{eqnarray}
            f^{k} = \mathop{\arg\min}\limits_{J(f) = \S_{k}} \L(\w) = \lambda \|f\| + \E_N(f) \label{eqn:f-update}
        \end{eqnarray}
    \ENDFOR
    \STATE {\bf Output} $f=f^{k-1}$
\end{algorithmic}
\end{algorithm}

The following theorem shows the performance guarantee of the solution obtained by Algorithm~\ref{alg:1} where its proof is given in Appendix A.

\begin{theorem} \label{lemma:1}
Let $f$ be the solution output from Algorithm~\ref{alg:1}. If $f$ is obtained by exiting from the middle of the loop, we have $\L(f) = \L(f^*)$. Otherwise, we have
\[
    \E_N(f) + \lambda \|f\| \leq \E_N(f^*) + \lambda\|f^*\| + \frac{2}{d-1}\|f^*\|^2.
\]
\end{theorem}

It should be emphasized that although the analysis in~\citep{shai-2010-trade} shows that the greedy coordinate descent 
approaches enjoy a geometric convergence rate when the objective function is both strongly convex and smooth in its variables, 
it can not be applied to our problem directly. This is because although the loss function $\ell(z, y)$ used in the regression 
is both strongly convex and smooth in the argument $y$, it is not strongly convex in $\{f_j\}_{j=1}^m$ because the prediction 
is given by $\sum_{j=1}^m f_j(\x)$. In next section, we present another approach for sparse MKL, based on greedy coordinate 
descent, that is able to achieve a geometric convergence rate under appropriate conditions.

\section{A Geometrically Convergent Algorithm for   Sparse MKL}
\label{sec:alg}
In this section, we present an algorithm for   sparse MKL that can achieve a geometric convergence rate under appropriate conditions.

We first argue that selecting kernel classifiers based on their functional norm may not necessarily be the best idea. This is because in order to ensure a removed kernel classifier $f_j$ to have a small impact on the overall regression error, we should be mostly concerned with $\mathrm{E}[|f_j(\x)|^2]$, instead of $\|f_j\|_{\H_j}$. To see this, we bound $\mathrm{E}[|f(\x) - y|^2] - \mathrm{E}[|f(\x) - f_j(\x) - y|^2]$, which measures the impact of removing $f_j$ from $f$
\begin{eqnarray*}
\mathrm{E}[|f(\x) - f_j(\x) - y|^2] - \mathrm{E}[|f(\x) - y|^2] & = & \mathrm{E}[|f_j(\x)|^2] - 2\mathrm{E}[f_j(\x)(f(\x) - y)] \\
& \leq & \mathrm{E}[|f_j(\x)|^2] + 2\sqrt{\mathrm{E}[|f_j(\x)|^2]}\sqrt{\mathrm{E}[|f(\x) - y|^2]}.
\end{eqnarray*}
Although $\|f_j\|_{\H_j} \geq \|f_j\|_{\infty} \geq \sqrt{\mathrm{E}[|f_j(\x)|^2]}$, there could be a significant gap 
between $\|f_j\|_{\H_j}$ and $\sqrt{\mathrm{E}[|f_j(\x)|^2]}$~\citep{smale-2007-learning}, making it possible for the functional norm based criterion to \textit{remove} the kernels that are \textit{important} in the final prediction. 

Based on the above discussion, we propose to measure the size of kernel classifiers $f_j$ by its $\ell_2$ norm, i.e., $\sqrt{\mathrm E|f_j(\x)|^2}$. Since the distribution of $\x$ is unavailable,  we introduce the empirical counterpart of $\sqrt{\mathrm E|f_j(\x)|^2}$, called empirical  $\ell_2$ norm and denoted by $\|f_j\|_{\ld}$. Given $f_j= \sum_{i=1}^N\alpha_{ji}\kappa_j(\x_i,\cdot)$, its $\ld$ norm is computed as
\begin{equation}
\|f_j\|_{\ld} = \sqrt{\frac{1}{N}\sum_{a=1}^N f^2_j(\x_a)} = \sqrt{\frac{1}{N}\sum_{a=1}^N \left(\sum_{b=1}^N \alpha_{jb}\kappa_j(\x_b, \x_a) \right)^2} = \frac{1}{\sqrt{N}}\|K_j\alpha_j\|_2,\label{eqn:l2}
\end{equation}
where $K_j = [\kappa_j(\x_a, \x_b)]_{N\times N}$ is the kernel matrix for $\kappa_j(\cdot, \cdot)$, and $\alpha_j=(\alpha_{j1},\cdots, \alpha_{jN})^{\top}$. For the purpose of our analysis, we also define an empirical $\ell_2$ norm for the combined classifier $f= \sum_{j=1}^mf_j=\sum_{j=1}^m\sum_{i=1}^N\alpha_{ji}\kappa_j(\x_i,\cdot)$ as
\begin{equation}
\|f\|_{\ld} = \sqrt{\frac{1}{N}\sum_{i=1}^Nf^2(\x_i)} = \frac{1}{\sqrt{N}}\left\|\sum_{j\in[m]}K_j\alpha_j\right\|_2.
\end{equation}

One way to exploit the empirical $\ell_2$ norm for sparse MKL is to incorporate it into (\ref{eqn:1}) as part of the regularization, leading to a mixture regularizer that is consisted of both $\|f_j\|_{\H_j}$ and $\|f_j\|_{\ld}$. A similar formulation is suggested in~\citep{Koltchinskii:2010:mkl}. It is however unclear as how to efficiently solve the related optimization problem to achieve a convergence rate better than $O(1/d)$. Instead, we will use the empirical $\ld$ norm to measure the size of gradients when performing greedy coordinate descent optimization. Our analysis in subsection~\ref{sec:conv} shows that this modification to Algorithm~\ref{alg:1}, together with other changes, will result in a geometric convergence rate under appropriate conditions, i.e.
\begin{equation*}
\E_N(f) - \min_{f\in \H} \E_N(f) \leq O(\max(0, (1-\tau)^d)),
\end{equation*}
where the value of $\tau$ will be  determined by analysis.

\begin{algorithm}[t]
\center \caption{A $\ld$ Norm based Greedy Coordinate Descent Approach for Sparse MKL}
\begin{algorithmic}[1] \label{alg:2}
    \STATE {\bf Input}: $\lambda > 0$: regularization parameter,  $d$: the number of selected kernels
    \STATE {\bf Initialization}: $f_j^{0} = 0, j \in [m]$ and $\S_{0} = \emptyset$.
    \FOR{$k = 1, \ldots, d$}
        \STATE $j_k = \mathop{\arg\max}_{j \in [m]} \left
        \|\nabla_j \E_N(f^{k-1})\right\|_{\ld}$
        \STATE Update the kernel classifier as
        \begin{eqnarray}
            f^{k} = f^{k-1} - f_{j_k}, \text{ where } f_{j_k}= \frac{1}{N}\sum_{i=1}^N a^k_i \kappa_{j_k}(\x_i, \cdot) \ \label{eqn:f-update-1}
        \end{eqnarray}
        where $\a^k = \left(\begin{array}{c}a^k_1\\ \ldots\\a^k_N\end{array}\right)$ is the projection of $\boldsymbol{\ell}'(f^{k-1})=\left(\begin{array}{c}\ell'(f^{k-1}\left(\x_1), y_1\right)\\ \ldots\\ \ell'\left(f^{k-1}(\x_N), y_N\right)\end{array}\right)$ into the space spanned by the column vectors of the kernel matrix $K_{j_k} = [\kappa_{j_k}(\x_a, \x_b)]_{N\times N}$.
    \ENDFOR
    \STATE {\bf Output} $f=f^d$
\end{algorithmic}
\end{algorithm}

Algorithm~\ref{alg:2} gives the basic steps of the new approach for sparse MKL. Similar to Algorithm~\ref{alg:1}, at each iteration, Algorithm~\ref{alg:2} chooses the kernel with the largest gradient and updates the kernel classifier based on the gradient with respect to the selected kernel. The key difference between these two algorithms is how to measure the size of the gradients. In Algorithm~\ref{alg:1}, the size of gradient $\nabla_j \E_N(f^{k})$ is measured by its functional norm, while Algorithm~\ref{alg:2} measures the size of gradient by $\ld$ norm of $\nabla_j \E_N(f^{k})$. In addition, Algorithm~\ref{alg:2} follows the idea of gradient descent for updating the kernel classifier $f^k$ and does not require solving any optimization problem. However, unlike the standard gradient descent algorithm that updates the classifier directly using the gradient, Algorithm~\ref{alg:2} projects the coefficients of $\nabla_{j_k}\E_N(f^{k-1})$ into the subspace spanned by the column vectors in $K_j$ before using it for updating. This step is critical for the correctness of the algorithm.

\subsection{Convergence Analysis}\label{sec:conv}
To analyze the performance of Algorithm~\ref{alg:2}, we assume there exists a sparse MKL solution that achieves a small regression error. More specifically, we slightly abuse our notation by redefining $f^*$ as the optimal kernel classifier that minimizes the empirical loss $\E_N(f)$,  $\fh$ as the optimal kernel classifier that minimizes the empirical loss using no more than $d$ kernels, and $\varepsilon^*$ be the difference in the empirical loss between $\fh$ and $f^*$, i.e.,
\begin{eqnarray}\label{eqn:def}
    f^* &=&\mathop{\arg\min}_{f\in \H} \E_N(f),\quad  \fh = \mathop{\arg\min}\limits_{f\in \H,\; |J(f)| \leq d} \E_N(f) \label{eqn:fd}, \quad \varepsilon^* = \E_N(\fh) - \E_N(f^*).
\end{eqnarray}
 We assume $\varepsilon^*$ is small, implying that the optimal solution $f^*$ can be well approximated by a function involved no more than $d$ kernels.

In order to state our result, we need to characterize the relationship among different kernel matrices. In~\citep{Koltchinskii:2011:oracle}, the author defines quantity $\beta(b, J, H)$ to capture the geometric relationship for a set of vectors $H = (\h_1, \ldots, \h_m) \in \R^{N\times m}$, i.e.,
\[
    \beta(b, J, H) = \inf\left\{\beta > 0: \sum_{j \in J} \lambda^2_j \leq \beta^2 \left\|\sum_{j=1}^m \lambda_j \h_j \right\|^2_2, \forall \lambda \in \C(b, J) \right\},
\]
where $b \geq 0$ is a nonnegative constant, $J \subset [m]$, and $\C(b, J)$ is defined as
\[
    \C(b, J) = \left\{ \lambda \in R^m: \sum_{j \notin J} \lambda^2_j \leq b^2\sum_{j \in J} \lambda^2_j \right\}.
\]
$\C(b, J)$ defines a set of sparse vector in which the components in $J$ dominates over the other components measured by their absolute values. When $b = 0$, vectors in $\C(b, J)$ only have non-zero elements in set $J$, leading to the standard definition of sparse vectors. $\beta(b, J, H)$ essentially captures the linearly dependence among vectors in $H$. For instance, when all $\h_j$ are normalized and orthogonal to each other, we have $\beta(0, J, H) = 1$. We extend $\beta(b, J, H)$ to $\beta(d, H)$ by taking into account all the vectors with no more than $d$ non-zero elements,
\[
    \beta(d, H) = \inf\{\beta(0, J, H): J \subset [m], |J| \leq d \}.
\]

We now generalize the above definitions to capture the ``dependence'' among the kernel matrices $\K = \{\Kh_1, \ldots, \Kh_m\}$, where $\Kh_j=K_j/N$. Since we need to deal with a sparse matrix $A = (\a_1, \ldots, \a_m) \in \R^{N\times m}$, we extend the definition of $\C(b, J)$ to $\S(b, J, \K)$ for sparse matrix as follows
\begin{eqnarray}
&&\S(b, J, \K) =\label{eqn:S}\\
&& \left\{ A = (\a_1, \ldots, \a_m) \in \R^{N\times m}: \sum_{j \notin J} \|\a_j\|_2 \leq b\sum_{j \in J} \|\a_j
\|_2, \a_j \in \span(K_j), j=1, \ldots, m \right\}, \nonumber
\end{eqnarray}
where $\span(K_j)$ stands for the subspace spanned by the column vectors of $K_j$. We then define quantity $\gamma(b, J, \K)$ to capture the ``dependence'' among matrices in $\K$
\begin{eqnarray}
    \gamma(b, J, \K) = \inf\left\{\gamma > 0: \sum_{j \in J} \|\a_j\|_2 \leq \gamma\left\|\sum_{j=1}^m \Kh_j\a_j \right\|_2, \forall A \in \S(b, J, \K) \right\}. \label{eqn:gamma-1}
\end{eqnarray}
We finally define $\gamma(d, \K)$ to take into account any matrix $A$ that has no more than $d$ non-zero column vectors
\begin{eqnarray}
    \gamma(d, \K) = \inf\left\{ \gamma(0, J, \K): J \subset [m], |J| \leq d \right\}. \label{eqn:gamma-2}
\end{eqnarray}
We note that the value of $\gamma(d,\K)$ is closely related to the correlation between the subspace spanned by any two matrices in $\K$.  For example, when subspaces spanned by each matrix $\Kh_j$ are orthogonal to each other and let the minimum non-zero eigenvalues of $\Kh_j, j\in[m]$ be larger than  $\sigma^+_{\min}\leq 1$, we have $\gamma(d,\K) \leq \sqrt{d}/\sigma^+_{\min}$.  More generally, if we let  $\delta(\K)$ denote the correlation between the subspace spanned by any two matrices in $\K$, defined as 

\[
    \delta(\K) = \max\limits_{1 \leq i < j \leq d} \max\limits_{\a_i, \a_j}\frac{|(\Kh_i \a_i)^{\top}(\Kh_j \a_j)|}{\|\Kh_i\a_i\|_2\|\Kh_j\a_j\|_2}.
\]
The following proposition shows the relationship between $\gamma(d,\K)$ and $\delta(\K)$ when $\delta(\K)$ is small.
\begin{proposition}
\label{prop:3}
If $\displaystyle \delta(\K)< \frac{1}{d-1}$,  the following inequality holds  for $\gamma(d,\K)$ and $\delta(\K)$,
\[
    \gamma(d, \K) \leq \frac{\sqrt{d} }{\sqrt{1 - (d-1)\delta(\K)}\sigma^+_{\min}},
\]
where $\sigma^+_{\min}$ is a lower bound of the minimum non-zero eigenvalues of $\Kh_j, j\in[m]$.
\end{proposition}

\textbf{Remark:} The correlation between different kernels has beed used in the previous studies for proving learning bounds for multiple kernel learning.
For example, in~\citep{Cortes:2009:LRL:1795114.1795128}, the authors derived generalization bounds  for  kernel ridge regression with $\ell_2$ regularization on  multiple kernels in  the case where the kernels are orthogonal.

The following lemma shows that when $\gamma(2d, \K)$ is bounded, the solution $f$ of the Algorithm~\ref{alg:2} converges to $f^*$ in a geometric rate.
\begin{lemma}\label{lemma:4}
Let $f$ be the solution output from Algorithm~\ref{alg:2}, and $(f^*, \fh, \varepsilon^*)$ be defined in~(\ref{eqn:def}).  For any $\mu \geq 1$, we have either $\E_N(f) - \E_N(f^*) \leq \mu(\E_N(\fh) - \E_N(f^*))$ or
\[
    \E_N(f) - \E_N(f^*)\leq \frac{1}{2}\left[\max(0, 1 - \tau)\right]^d,
\]
where $\tau$ is defined as
\[
    \tau = \frac{(\mu - 1)^2}{8\mu(\mu + 1)\gamma(2d, \K)}.
\]
\end{lemma}
The proof is deferred to Appendix B.

As indicated by Lemma~\ref{lemma:4}, Algorithm~\ref{alg:2} achieves a geometric convergence rate of $(1 - \tau)^d$, where $\tau$ depends on the parameter $\gamma(2d, \K)$. In particular, the smaller the $\gamma(2d, \K)$, the faster the convergence. One shortcoming with Lemma~\ref{lemma:4} is that it does not give the explicit expression for bounding $\E_N(f) - \E_N(f^*)$ because the bound depends on parameter $\mu$. The following theorem makes the bound more explicit.
\begin{theorem} \label{cor:1}
Let $f$ be the solution output from Algorithm~\ref{alg:2}, and $(f^*, \varepsilon^*)$ be defined in~(\ref{eqn:def}).
 If the number of selected kernels $d$ is sufficiently large, i.e.,
\[
    d  \geq 16\gamma(2d, \K) \ln\left(\frac{1}{12\varepsilon^*}\right),
\]
then  we have
\[
    \E_N(f) - \E_N(f^*) \leq6\varepsilon^*.
\]
\end{theorem}
\begin{proof}
According to Lemma~\ref{lemma:4}, we have
\[
    \E_N(f) - \E_N(f^*) \leq \min\limits_{\mu \geq 1} \max\left(\mu \varepsilon_*, \frac{1}{2}\left[\max(0, 1 - \tau)\right]^d  \right).
\]
It is straightforward to show that for any $z \in [0, 1)$, if $\mu \geq (2+z)/(1 - z)$, we have $\tau > z/[8\gamma]$, where $\gamma = \gamma(2d, \K)$. We thus have
\[
    \E_N(f) - \E_N(f^*) \leq \min\limits_{z \in [0, 1)} \max\left(\frac{3\varepsilon_*}{1-z} , \frac{1}{2}\exp(-dz/[8\gamma]) \right)\leq \min\limits_{z \in [0, 1)} \max\left(\frac{3\varepsilon_*}{1-z} , \frac{1}{2}\exp\left(2z\ln(12\varepsilon^*)\right)  \right).
\]
The optimum of R.H.S is achieved when
\[
    \frac{3\varepsilon_*}{1-z} =  \frac{1}{2}\exp\left(2z\ln(12\varepsilon^*)\right).
\]
Under the condition given in the theorem, we have the above equation satisfied if $z = 1/2$. We also note that the solution to the above equation is unique because $ \frac{3\varepsilon_*}{1-z} - \frac{1}{2}\exp\left(2z\ln(12\varepsilon^*)\right)$ is monotonically increasing in $z$. We complete the proof by plugging $z = 1/2$.
\end{proof}

\subsection{Generalization Bound}
As previously mentioned, there is a rich body of literature dealing with the generalization error bounds of MKL algorithms~\citep{hussain:2011:note,Ying:2007:mkl,Bousquet:2003:mkl,Srebro:2006:mkl,Ying:2009:mkl}. In the remarkable work of~\citep{Lanckriet:2004:LKM}, a convergence rate of  $O(\sqrt{m/N})$ has been proved for MKL with $\ell_1$ constraint.  After that, this bound is improved utilizing the pseudo-dimension of the given kernel class  in~\citep{Srebro:2006:mkl}. \citet{Cortes:2009:LRL:1795114.1795128} studied the problem of multiple kernel learning with $\ell_2$ regularization  for regression, and derived learning bounds  that have an additive term $O(\sqrt{m/N})$ when kernels are orthogonal. In~\citep{Cortes:2010:mkl} new generalization bounds for the family of convex combination of kernel function with $\ell_1$ constraint were presented which have logarithmic dependency on the number of kernels (i.e., $\sqrt{\ln m}$).  It is worth mentioning that although the mentioned generalization bounds differ in their dependency on the number of base kernels, however, all convergence rate presented  are of order $1/\sqrt{N}$ with respect to the number $N$ of samples.  It is worth mentioning that although the mentioned
generalization bounds differ in their dependency on the number of base kernels, however,
all convergence rate presented are of order $1/\sqrt{N}$ with respect to the number $N$ of samples. Recently, \citep{Marius-2011-gen-bound}  utilized local Rademacher complexity and derived
a tighter upper bound with respect to $N$ for $\ell_p$ norm MKL by considering the decay rate of eigenvalues of kernel matrices. \cite{Suzuki-2011-gen-bound} presented a unified framework to derive the bounds  of MKL with arbitrary mixed-norm type regularization.

To present the generalization error bound for the sparse MKL solution obtained by Algorithm~\ref{alg:2}, we introduce the  following bounded RKHS  $\H(R)$  as
\[
\H(R) = \left\{f = \sum_{j=1}^m f_j: f_j \in \H_j, j \in [m], \sum_{j=1}^m \|f_j\|_{\H_j} \leq R \right\}.
\]
The generalization error bound is stated in the following theorem.
\begin{theorem}
\label{thm:2}
Let $f$ be the solution output from Algorithm~\ref{alg:2}, $(f^*, \varepsilon^*)$ be defined in~(\ref{eqn:def}),
and  $f^*_R$ be the optimal function for minimizing the expected loss in $\H(R)$, i.e. $f^*_R = \arg\min\limits_{f\in \H(R)} \E(f)$. 
Assuming $A > 1$, $m \geq 3$, and $A\ln(m+1) \leq N \leq 2^{m+1}$,  we have either $\|f - f^*_R \| \leq 8\max(R, \sqrt{d})/\sqrt{N}$ or with a probability at least $1 - (m+1)^{-A + 1}$,
\[
    \E(f) - \E(f^*_R)\leq \E_N(f) - \E_N(f^*)+ 196(R+\sqrt{d})^2\sqrt{\frac{A\ln(m+1)}{N}}.
\]
Under the assumption $
   d \geq 16\gamma(2d, \K) \ln\left(\frac{1}{12\varepsilon^*}\right)
$, 
we have
\[
    \E(f) - \E(f^*_R ) \leq 6\varepsilon^* + 196(R+\sqrt{d})^2\sqrt{\frac{A\ln(m+1)}{N}}.
\]
\end{theorem}
\textbf{Remark:} First, we should note that there is a tradeoff in the generalization bound with respect to $d$, since $\varepsilon^*$ could increase when $d$ decreases. Second, the generalization bound of the proposed algorithm for learning a combination of no more than $d$ kernels has an additive term $O(d\sqrt{\ln m/N})$, which deteriorates by a factor of $d$ compared to previous learning bounds of MKL.  Third,  if  we assume $\varepsilon^*$ is small, e.g., in the order of  $O(N^{-1/2})$, and  $\gamma(2d,\K)\leq O(\sqrt{d})$,  we can let $d = O(\ln^2 N)$, i.e. learning a combination of no more than $O(\ln^2 N)$ kernels, and we have the generalization error of the proposed algorithm bounded by $O(\ln^2 N\sqrt{\ln m/ N})$ , which only deteriorates by a factor of $\ln^2N$ compared with  the best  known learning bound of MKL (i.e. $O(\sqrt{\ln m/N})$). 

In order to prove Theorem~\ref{thm:2}, we need the following lemma to bound the concentration of regression error, where $(\ell\circ f)(\x, y)= \ell(f(\x), y)$, and $P_N$ and  $P$ are defined by 
\begin{eqnarray*}
P_N(F)& =& \frac{1}{N}\sum_{i=1}^NF(\x_i, y_i),\quad P(F) =  \mathrm E_{\x, y}[F(\x, y)], 
\end{eqnarray*}
for any function $F$ that takes $(\x, y)$ as input.

\begin{lemma} \label{lemma:gen-bound-1}
Define $r_0 = 8R/\sqrt{N}$ and $L = R+1$. Let $g \in \H(R)$ be a fixed function. Assume $A > 1$, and $A\ln(m+1) \leq N \leq 2^{m+1}$. With a probability at least $1 - (m+1)^{-A + 1}$, for any $f \in \H(R)$, and any $r > r_0$, we have
\begin{eqnarray*}
\sup\limits_{\sum_{i=1}^m \|f_i - g_i\|_{\H_j} \leq r} |(P - P_N)(\ell \circ f - \ell \circ g)| \leq 88Lr\sqrt{\frac{A\ln(m+1)}{N}}.
\end{eqnarray*}
\end{lemma}
The proof of Lemma~\ref{lemma:gen-bound-1} is provided in Appendix D.  We are now ready to prove Theorem~\ref{thm:2}.
\begin{proof}[of Theorem~\ref{thm:2}]
First, we show that the solution $f$ obtained by Algorithm~\ref{alg:2} has a bounded functional norm $\|f\|$. We have
\begin{eqnarray*}
\|f\| = \left\|\sum_{k=1}^d f_{j_k} \right\| \leq \sum_{k=1}^d \|f_{j_k}\|_{\H_{j_k}}.
\end{eqnarray*}
Following inequality~(\ref{eqn:norm}) in the Proof of Lemma~\ref{lemma:4}, we have
\begin{equation*}
\|f_{j_k}\|_{\H_{j_k}} = \frac{1}{N^2}\a^kK_{j_k}\a^k= \|\nabla_{j_k}\E_N(f^k)\|_{\H_{j_k}}.
\end{equation*}
According to the inequality in~(\ref{eqn:norm2}) in the Proof of Lemma~\ref{lemma:4}, we have
\[
    \|f_{j_k}\|^2_{\H_{j_k}} \leq 2\left(\E_N(f^{k-1}) - \E_N(f^{k})\right),
\]
due to $ \|\nabla_{j_k}\E_N(f^k)\|_{\ld}\leq  \|\nabla_{j_k}\E_N(f^k)\|_{\H_{j_k}}$.
Hence
\[
\|f\| \leq \sum_{k=1}^d \|f_{j_k}\|_{\H_{j_k}} \leq \sqrt{d}\sqrt{\sum_{k=1}^d \|f_{j_k}\|^2_{\H_{j_k}}}\leq \sqrt{2d\E_N(f^0)}\leq \sqrt{\frac{d}{N}}\|\y\|_2\leq \sqrt{d}.
\]
Second, we have

\begin{eqnarray*}
\E(f) & \leq & \E(f^*_R)  +\E_N(f) - \E_N(f^*_R)+\E(f) - \E_N(f)+ \E_N(f^*_R) - \E(f^*_R)\\
&\leq &\E(f^*_R )+  \E_N(f) - \E_N(f_R^*)  + \sup\limits_{f \in \H(\sqrt{d})} |(P - P_N)(\ell\circ f - \ell\circ f^*_R )|\\
&\leq &\E(f^*_R )+ \E_N(f) - \E_N(f^*) + \sup\limits_{f \in \H(\sqrt{d})} |(P - P_N)(f - f^*_R )|.
\end{eqnarray*}

Using the Lemma~\ref{lemma:gen-bound-1}, we have either $\|f - f^*_R\| \leq 8\max(R, \sqrt{d})/\sqrt{N}$, or with a probability at least $1 - (m+1)^{-A + 1}$, that
\begin{eqnarray*}
\sup\limits_{f \in \H(\sqrt{d})} |(P - P_N)(\ell\circ f -\ell\circ f^*_R )| &\leq&\sup\limits_{\|f-g\|\leq R+\sqrt{d}} |(P - P_N)(\ell\circ f -\ell\circ g)|\\
&\leq & 88(\max(R,\sqrt{d}) +1 )(R+\sqrt{d})\sqrt{\frac{A\ln(m+1)}{N}}\\
&\leq & 196(R+\sqrt{d})^2\sqrt{\frac{A\ln(m+1)}{N}},
\end{eqnarray*}
leading to
\[
\E(f) \leq\E(f^*_R )+ \E_N(f) - \E_N(f^*) + 196(R+\sqrt{d})^2\sqrt{\frac{A\ln(m+1)}{N}}.
\]
We complete the proof by plugging the result from Theorem~\ref{cor:1}.
\end{proof}

\section{Conclusion}
\label{sec:conc}
In this paper, we developed an efficient algorithm for sparse multiple kernel learning (MKL) based on greedy coordinate descent 
algorithm. By using an empirical $\ell_2$ norm for measuring the size of functional gradients, we are able to achieve 
a geometric convergence rate under certain conditions. We also prove the generalization error bound of the proposed algorithm.  As the future work, we plan to provide better 
quantization about the independence among kernel matrices, a key condition for our algorithm to achieve geometric  convergence.  

\bibliography{mkl}

\begin{thebibliography}{33}
\providecommand{\natexlab}[1]{#1}
\providecommand{\url}[1]{\texttt{#1}}
\expandafter\ifx\csname urlstyle\endcsname\relax
  \providecommand{\doi}[1]{doi: #1}\else
  \providecommand{\doi}{doi: \begingroup \urlstyle{rm}\Url}\fi

\bibitem[Argyriou et~al.(2005)Argyriou, Micchelli, and
  Pontil]{argyriou:2005:learning}
Andreas Argyriou, Charles~A. Micchelli, and Massimiliano Pontil.
\newblock Learning convex combinations of continuously parameterized basic
  kernels.
\newblock In \emph{Proceedings of the 18th Annual Conference on Learning
  Theory}, pages 338--352, 2005.

\bibitem[Argyriou et~al.(2006)Argyriou, Hauser, Micchelli, and
  Pontil]{Argyriou:2006:DC}
Andreas Argyriou, Raphael Hauser, Charles~A. Micchelli, and Massimiliano
  Pontil.
\newblock A dc-programming algorithm for kernel selection.
\newblock In \emph{Proceedings of the 23rd international conference on Machine
  learning}, pages 41--48, 2006.

\bibitem[Bach(2008)]{Bach:2008:mkl}
Francis Bach.
\newblock Exploring large feature spaces with hierarchical multiple kernel
  learning.
\newblock In \emph{Proceedings of the 22nd Annual Conference on Neural
  Information Processing Systems}, pages 105--112, 2008.

\bibitem[Bach et~al.(2004)Bach, Lanckriet, and Jordan]{Bach:2004:MKL}
Francis~R. Bach, Gert R.~G. Lanckriet, and Michael~I. Jordan.
\newblock Multiple kernel learning, conic duality, and the smo algorithm.
\newblock In \emph{Proceedings of the 21st International Conference on Machine
  learning}, pages 6--13, 2004.

\bibitem[Bartlett et~al.(2002)Bartlett, Bousquet, and
  Mendelson]{Bartlett:2002:localrademacher}
Peter~L. Bartlett, Olivier Bousquet, and Shahar Mendelson.
\newblock Local rademacher complexities.
\newblock \emph{Annals of Statistics}, pages 44--58, 2002.

\bibitem[Bousquet and Herrmann(2003)]{Bousquet:2003:mkl}
Olivier Bousquet and Daniel J.~L. Herrmann.
\newblock On the complexity of learning the kernel matrix.
\newblock In \emph{Proceedings of the 17th Annual Conference on Neural
  Information Processing Systems}, pages 399--406, 2003.

\bibitem[Cortes et~al.(2009)Cortes, Mohri, and
  Rostamizadeh]{Cortes:2009:LRL:1795114.1795128}
Corinna Cortes, Mehryar Mohri, and Afshin Rostamizadeh.
\newblock L2 regularization for learning kernels.
\newblock In \emph{Proceedings of the 25th Conference on Uncertainty in
  Artificial Intelligence}, pages 109--116, 2009.

\bibitem[Cortes et~al.(2010)Cortes, Mohri, and Rostamizadeh]{Cortes:2010:mkl}
Corinna Cortes, Mehryar Mohri, and Afshin Rostamizadeh.
\newblock Generalization bounds for learning kernels.
\newblock In \emph{Proceedings of the 27th Internationl Conference on Machine
  Learning}, 2010.

\bibitem[Hussain and Shawe-Taylor(2011)]{hussain:2011:note}
Zakria Hussain and John Shawe-Taylor.
\newblock A note on improved loss bounds for multiple kernel learning.
\newblock \emph{CoRR}, abs/1106.6258, 2011.

\bibitem[Kloft and Blanchard(2011)]{Marius-2011-gen-bound}
Marius Kloft and Gilles Blanchard.
\newblock The local rademacher complexity of lp-norm multiple kernel learning.
\newblock In \emph{Proceedings of the 25th Annual Conference on Neural
  Information Processing Systems}, 2011.

\bibitem[Kloft et~al.(2009)Kloft, Brefeld, Sonnenburg, Laskov, M\"{u}ller, and
  Zien]{kloft-2009-efficient}
Marius Kloft, Ulf Brefeld, Soeren Sonnenburg, Pavel Laskov, Klaus-Robert
  M\"{u}ller, and Alexander Zien.
\newblock Efficient and accurate lp-norm multiple kernel learning.
\newblock In \emph{Proceedings of the 23rd Annual Conference on Neural
  Information Processing Systems}, pages 997--1005. 2009.

\bibitem[Koltchinskii(2011)]{Koltchinskii:2011:oracle}
Vladimir Koltchinskii.
\newblock \emph{Oracle Inequalities in Empirical Risk Minimization and Sparse
  Recovery Problems}.
\newblock Springer, 2011.

\bibitem[Koltchinskii and Yuan(2008)]{Koltchinskii:2008:sparsity}
Vladimir Koltchinskii and Ming Yuan.
\newblock Sparse recovery in large ensembles of kernel machines on-line
  learning and bandits.
\newblock In \emph{Proceedings of the 21st Annual Conference on Learning
  Theory}, pages 229--238, 2008.

\bibitem[Koltchinskii and Yuan(2010)]{Koltchinskii:2010:mkl}
Vladimir Koltchinskii and Ming Yuan.
\newblock Sparsity in multiple kernel learning.
\newblock \emph{Annuals of Statistics}, 38:\penalty0 3660--3694, 2010.

\bibitem[Lanckriet et~al.(2004)Lanckriet, Cristianini, Bartlett, Ghaoui, and
  Jordan]{Lanckriet:2004:LKM}
Gert R.~G. Lanckriet, Nello Cristianini, Peter Bartlett, Laurent~El Ghaoui, and
  Michael~I. Jordan.
\newblock Learning the kernel matrix with semidefinite programming.
\newblock \emph{Journal of Machine Learning Research}, 5:\penalty0 27--72,
  December 2004.

\bibitem[Lewis et~al.(2006)Lewis, Jebara, and Noble]{Lewis:2006:NKC}
Darrin~P. Lewis, Tony Jebara, and William~Stafford Noble.
\newblock Nonstationary kernel combination.
\newblock In \emph{Proceedings of the 23rd International Conference on Machine
  Learning}, pages 553--560, 2006.

\bibitem[Micchelli and Pontil(2005)]{Micchelli:2005:mkl}
Charles~A. Micchelli and Massimiliano Pontil.
\newblock Learning the kernel function via regularization.
\newblock \emph{Journal of Machine Learning Research}, 6:\penalty0 1099--1125,
  2005.

\bibitem[Nesterov(2010)]{nesterov-2007}
Yurii Nesterov.
\newblock Efficiency of coordinate descent methods on huge-scale optimization
  problems.
\newblock CORE Discussion Paper \#2010-2, 2010.

\bibitem[Ong et~al.(2005)Ong, Smola, and Williamson]{Ong:2005:mkl}
Cheng~Soon Ong, Alexander~J. Smola, and Robert~C. Williamson.
\newblock Learning the kernel with hyperkernels.
\newblock \emph{Journal of Machine Learning Research}, 6:\penalty0 1043--1071,
  December 2005.

\bibitem[Orabona and Jie(2011)]{ICML2011:Orabona:mkl}
Francesco Orabona and Luo Jie.
\newblock Ultra-fast optimization algorithm for sparse multi kernel learning.
\newblock In \emph{Proceedings of the 28th International Conference on Machine
  Learning}, pages 249--256, 2011.

\bibitem[Rakotomamonjy et~al.(2008)Rakotomamonjy, Bach, Canu, and
  Grandvalet]{Rakotomamonjy:2008:mkl}
A.~Rakotomamonjy, F.~R. Bach, S.~Canu, and Y.~Grandvalet.
\newblock Simplemkl.
\newblock \emph{Journal of Machine Learning Research}, 9:\penalty0 2491--2521,
  2008.

\bibitem[Shalev-Shwartz et~al.(2010)Shalev-Shwartz, Srebro, and
  Zhang]{shai-2010-trade}
Shai Shalev-Shwartz, Nathan Srebro, and Tong Zhang.
\newblock Trading accuracy for sparsity in optimization problems with sparsity
  constraints.
\newblock \emph{SIAM Journal on Optimization}, 20\penalty0 (6):\penalty0
  2807--2832, 2010.

\bibitem[Sindhwani and Lozano(2011)]{omp-mkl-nips-2011}
Vikas Sindhwani and Aurelie~C. Lozano.
\newblock Non-parametric group orthogonal matching pursuit for sparse learning
  with multiple kerenels.
\newblock In \emph{Proceedings of the 25th Annual Conference on Neural
  Information Processing Systems}, 2011.

\bibitem[Smale and Zhou(2007)]{smale-2007-learning}
Steve Smale and Ding-Xuan Zhou.
\newblock Learning theory estimates via integral operators and their
  approximations.
\newblock \emph{Constructive Approximation}, 26:\penalty0 153--172, 2007.

\bibitem[Sonnenburg et~al.(2006)Sonnenburg, R\"{a}tsch, Sch\"{a}fer, and
  Sch\"{o}lkopf]{Sonnenburg:2006:LSM}
S\"{o}ren Sonnenburg, Gunnar R\"{a}tsch, Christin Sch\"{a}fer, and Bernhard
  Sch\"{o}lkopf.
\newblock Large scale multiple kernel learning.
\newblock \emph{Journal of Machine Learning Research}, 7:\penalty0 1531--1565,
  2006.

\bibitem[Srebro and Ben-david(2006)]{Srebro:2006:mkl}
Nathan Srebro and Shai Ben-david.
\newblock Learning bounds for support vector machines with learned kernels.
\newblock In \emph{Proceedings of the 19th Annual Conference on Learning
  Theory}, pages 169--183, 2006.

\bibitem[Suzuki(2011)]{Suzuki-2011-gen-bound}
Taiji Suzuki.
\newblock Unifying framework for fast learning rate of non-sparse multiple
  kernel learning.
\newblock In \emph{Proceedings of the 25th Annual Conference on Neural
  Information Processing Systems}, 2011.

\bibitem[Suzuki and Tomioka(2011)]{citeulike:9417296}
Taiji Suzuki and Ryota Tomioka.
\newblock {SpicyMKL}: a fast algorithm for multiple kernel learning with
  thousands of kernels.
\newblock \emph{Machine Learning}, pages 1--32, 2011.

\bibitem[Vishwanathan et~al.(2010)Vishwanathan, sun, Ampornpunt, and
  Varma]{Vishwanathan-2010-mkl}
S.~V.~N. Vishwanathan, Zhaonan sun, Nawanol Ampornpunt, and Manik Varma.
\newblock Multiple kernel learning and the smo algorithm.
\newblock In \emph{Proceedings of the 24th Annual Conference on Neural
  Information Processing Systems}, pages 2361--2369, 2010.

\bibitem[Xu et~al.(2008)Xu, Jin, King, and Lyu]{xu:2008:level}
Z.~Xu, R.~Jin, I.~King, and M.~R. Lyu.
\newblock An extended level method for efficient multiple kernel learning.
\newblock In \emph{Proceedings of the 22nd Annual Conference on Neural
  Information Processing Systems}, pages 1825--1832, 2008.

\bibitem[Ying and Campbell(2009)]{Ying:2009:mkl}
Yiming Ying and Colin Campbell.
\newblock Generalization bounds for learning the kernel.
\newblock In \emph{Proceedings of the 22nd Annual Conference on Learning
  Theory}, 2009.

\bibitem[Ying and Zhou(2007)]{Ying:2007:mkl}
Yiming Ying and Ding-Xuan Zhou.
\newblock Learnability of gaussians with flexible variances.
\newblock \emph{Journal of Machine Learning Research}, 8, December 2007.

\bibitem[Yun et~al.(2011)Yun, Tseng, and Toh]{tseng-2011-bcd}
Sangwoon Yun, Paul Tseng, and Kim-Chuan Toh.
\newblock A block coordinate gradient descent method for regularized convex
  separable optimization and covariance selection.
\newblock \emph{Mathematical Programming}, 129\penalty0 (2):\penalty0 331--355,
  2011.

\end{thebibliography}

\section*{Appendix A. [Proof of Theorem~\ref{lemma:1}]}

First, we bound the difference between $\L(f^k)$ and $\L(f^*)$ and show that for $k \geq 1$, the following holds
\begin{eqnarray}
\label{eqn:lf}
\L(f^{k+1}) - \L(f^*) \leq \frac{2\|f^*\|^2}{k}.
\end{eqnarray}

Similar to the standard theory of greedy algorithm~\citep{shai-2010-trade}, we have
\begin{eqnarray*}
    \L(f^k) - \L(f^*) \leq \sum_{j=1}^m \left\langle f_j^k - f_j^*, \nabla_j \E_N(f^k) + \lambda \delta_j \right\rangle_{\H_j},
\end{eqnarray*}
where $\delta_j \in \partial_j \|f_j^k\|_{\H_j}$. Since $f^k$ is the optimal solution of $\E_N(f) + \lambda\|f\|$  on the support $J(f^k)$, we have $\nabla_j \E_N(f^k) + \lambda \partial_j\|f_j^k\|_{\H_j} = 0, \forall j\in J(f^k)$. By choosing
\[
    \delta_j = -\frac{\nabla_j \E_N(f^k)}{\max(\lambda,
    \|\nabla_j \E_N(f^k)\|_{\H_j})}, j \notin J(f^k),
\]
we have
\begin{eqnarray*}
    \L(f^k) - \L(f^*) & \leq & \sum_{j \notin J(f^k)} \left\langle -f^*_j, \nabla_j \E_N(f^k) + \lambda \delta_j \right\rangle_{\H_j}  \leq  \|f^*\|\left[\max_{j \in [m]} |\nabla_j \E_N(f^k)|_{\H_j} - \lambda\right]_+,
\end{eqnarray*}
where $[z]_+ = \max(0, z)$. The above inequality indicates that if $\max_{j\in[m]}\|\nabla_j\E_N(f^k)\|_{\H_j}\leq \lambda$, $f^k$ is the optimal solution, we thus exist the loop. 

In the following, we assume $\max_{j\in[m]}\|\nabla_j\E_N(f^k)\|_{\H_j}>\lambda$.  We have
\begin{eqnarray*}
\L(f^{k+1})&=& \min\limits_{J(f)= \S_{k+1}}\E_N(f) + \lambda \|f\|\\
 & \leq & \min\limits_{J(f) = \S_{k+1}} \E_N(f^k) + \lambda \|f\| + \sum_{j=1}^m
\left\langle f_j - f_j^k, \nabla_j \E_N(f^k)\right\rangle_{\H_j} + \frac{1}{2N}\sum_{i=1}^N(f(\x_i) - f^k(\x_i))^2,
\end{eqnarray*}
where the inequality follows the definition of $\E_N(f)$.
To bound the R.H.S., we consider the following construction of $f$
\[
    f = f^k - \eta g_{j_{k+1}} = f^k - \eta \frac{\nabla_{j_{k+1}} \E_N(f^k)}{\|\nabla_{j_{k+1}} \E_N(f^k)\|_{\H_j}}.
\]
Using the above solution $f$, we have
\begin{eqnarray*}
\L(f^{k+1}) & \leq& \L(f^k)  + \eta\lambda - \eta \|\nabla_{j_{k+1}} \E_N(f^k)\|_{\H_j}  + \frac{\eta^2}{2 N}\sum_{i=1}^N [g_{j_{k+1}}(\x_i)]^2.
\end{eqnarray*}
Since the above inequality hold for any $\eta\geq 0$ and $j_{k+1} =\arg\max_j \|\nabla_j\E_N(f^k)\|$, we have
\begin{eqnarray*}
\L(f^{k+1}) & \leq & \L(f^k) + \min\limits_{\eta \geq 0} - \eta \left(\max_{j\in[m]}\|\nabla_j \E_N(f^k)\|_{\H_j} - \lambda \right) + \frac{\eta^2}{2 N}\sum_{i=1}^N [g_{j_{k+1}}(\x_i)]^2 \\
& \leq & \L(f^k) + \min\limits_{\eta \geq 0} - \eta \left(\max_{j\in[m]}\|\nabla_j \E_N(f^k)\|_{\H_j} - \lambda \right) + \frac{\eta^2}{2} \\
& \leq & \L(f^k) - \frac{1}{2}\left[\max_{j \in [m]} \|\nabla_j \E_N(f^k)\|_{\H_j} - \lambda \right]_+^2,
\end{eqnarray*}
where the second step follows $\|g_{j_{k+1}}\|_{\H_j} \leq 1$ and therefore $|g_{j_{k+1}}(\x_i)| \leq 1$ since $\kappa_{j}(\x_i, \x_i) \leq 1$.
As a result, when $\max_{j \in [m]} \|\nabla_j \E_N(f^k)\|_{\H_j} - \lambda > 0$, we have
\[
\L(f^k) - \L(f^{k+1}) \geq \frac{\left(\L(f^k) - \L(f^*)\right)^2}{2 \|f^*\|^2}.
\]
Define $\epsilon_k = \L(f^k) - \L(f^*)$. We have
\[
\frac{1}{\epsilon_{k+1}} - \frac{1}{\epsilon_{k}} \geq \frac{\L(f^k) - \L(f^{k+1})}{\epsilon_k^2} \geq
\frac{1}{2 \|f^*\|^2},
\]
leading to the result in (\ref{eqn:lf}).

Next, we consider two cases. In the first case, if $f$ is obtained in the middle of the loop, we have $\max_{j \in [m]} \|\nabla_j \E_N(f)\|_{\H_j} \leq \lambda$, and therefore have $\L(f) = \L(f^*)$. If $f$ is obtained by finishing all the loops, using (\ref{eqn:lf}) , we have the desired rate as 
\[
    \L(f) - \L(f^*) \leq \frac{2 \|f^*\|^2}{d-1}.
\]



\section*{Appendix B. [Proof of Lemma~\ref{lemma:4}]}
\label{app:b}
Similar to the proof of Theorem~\ref{lemma:1}, we have
\begin{eqnarray*}
    \E_N(f^k) - \E_N(\fh) \leq \sum_{j=1}^m \left\langle f_j^k - \fh_j, \nabla_j \E_N(f^k) \right\rangle_{\H_j}.
\end{eqnarray*}
According to the representer theorem, we have
\[
    f_j^k (\x) = \sum_{i=1}^m \alpha_{j,i}^k \kappa_j(\x_i, \x), \quad \fh_j(\x) = \sum_{j=1}^m \hat\alpha_{j,i} \kappa_j(\x_i, \x),
\]
where $\alpha_j^k = (\alpha_{j,1}^k, \ldots, \alpha_{j,n}^k)^{\top} \in \R^n$ and $\hat\alpha_j = (\hat\alpha_{j,1}, \ldots, \hat\alpha_{j,n})^{\top}\in \R^n$ are vector representation of function $f_j^k$ and $\fh_j$. Due to the projection step in updating the kernel classifier (step 5 in Algorithm~\ref{alg:2}), we have $\alpha_j^k \in \span(K_j)$. It is also safe to assume $\hat\alpha_j \in \span(K_j)$ because otherwise we can always project $\hat\alpha_j$ into the subspace $\span(K_j)$ without changing the value $\fh_j(\x_i), i\in[N]$, and therefore without change $\E_N(\fh)$. We define a norm $\|\cdot\|_a$ as
\[
    \|f_j^k\|_a = \sqrt{N}\|\alpha_j^k\|_2, \; \|\fh_j\|_a = \sqrt{N} \|\hat\alpha_j\|_2.
\]
Using these notations, we rewrite $\E_N(f^k) - \E_N(\fh)$ as
\begin{eqnarray*}
    \E_N(f^k) - \E_N(\fh) & \leq & \sum_{j=1}^m \left\langle f_j^k - \fh_j, \nabla_j \E_N(f^k) \right\rangle_{\H_j} \leq \sum_{j=1}^m \|f_j^k - \fh_j\|_a \left\| \nabla_j \E_N(f^k) \right\|_{\ld} \\
    & \leq & \left(\sum_{j=1}^m \|f_j^k - \fh_j\|_a \right)\max\limits_{1 \leq j \leq m} \left\| \nabla_j \E_N(f^k) \right\|_{\ld},
\end{eqnarray*}
where the second inequality follows from Cauchy inequality and the definition of $\ld$ norm of $\left\| \nabla_j \E_N(f^k) \right\|_{\ld}$ that is given by
\[
\left\| \nabla_j \E_N(f^k) \right\|^2_{\ld} = \frac{1}{N} \sum_{a=1}^N \left(\frac{1}{N}\sum_{b=1}^N \ell'(f^k(\x_b), y_b) \kappa_j(\x_a, \x_b) \right)^2 = \frac{1}{N}\|K_j\boldsymbol\ell'(f^k)/N\|_2^2,
\]
where $\boldsymbol\ell'(f^k)= (\ell'(f^k(\x_1), y_1),\cdots, \ell'(f^k(\x_N), y_N))^{\top}$. Using the following equality
\[
    f^{k+1} = f^k - \frac{1}{N}\sum_{i=1}^N a^{k+1}_i \kappa_{j_{k+1}}(\x_i, \cdot),
\]
where $\a^{k+1} = (a^{k+1}_1, \ldots, a^{k+1}_N)^{\top}$ is the projection of vector $\boldsymbol\ell'(f^k)$ into the subspace $\span(K_{j_{k+1}})$, we have
\begin{eqnarray}
\E_N(f^{k+1}) & \leq & \E_N(f^k) + \sum_{j=1}^m
\left\langle f^{k+1}_j - f_j^k, \nabla_j \E_N(f^k)\right\rangle_{\H_j} + \frac{1}{2N}\sum_{i=1}^N(f(\x_i) - f^k(\x_i))^2 \nonumber\\
& = & \E_N(f^k) - \|\nabla_{j_{k+1}} \E_N(f^k)\|_{\H_{j_{k+1}}}^2 + \frac{1}{2}\|\nabla_{j_{k+1}}\E_N(f^k)\|^2_{\ld}\label{eqn:norm2}\\
 &\leq& \E_N(f^k) - \frac{1}{2}\|\nabla_{j_{k+1}}\E_N(f^k)\|^2_{\ld},\nonumber
\end{eqnarray}
where we use  $f^{k+1}_j=f_j^k, \forall j\neq j_{k+1}$,
\begin{eqnarray}\label{eqn:norm}
\left\langle f^{k+1}_j - f_j^k, \nabla_j \E_N(f^k)\right\rangle_{\H_j}&=& - \frac{1}{N^2}{\a^{k+1}}^{\top} K_{j_{k+1}}\boldsymbol\ell'(f^k) =- \frac{1}{N^2}{\boldsymbol\ell'(f^k)}^{\top} K_{j_{k+1}}\boldsymbol\ell'(f^k)\nonumber\\
&=& - \|\nabla_{j_{k+1}} \E_N(f^k)\|_{\H_j}^2.
\end{eqnarray}
\begin{equation}
 \frac{1}{N}\sum_{i=1}^N(f(\x_i) - f^k(\x_i))^2=\frac{1}{N}\|K_{j_{k+1}}\a^k/N\|_2^2=\frac{1}{N}\|K_{j_{k+1}}\boldsymbol\ell'(f^k)/N\|_2^2= \left\| \nabla_{j_{k+1}} \E_N(f^k) \right\|^2_{\ld},\nonumber
\end{equation}
and the fact $\|\nabla_j\E_N(f^k)\|_{\ld} \leq \|\nabla_j\E_N(f^k)\|_{\H_j}$. As a result, we have
\[
\E_N(f^k) - \E_N(f^{k+1}) \geq \frac{\left(\E_N(f^k) - \E_N(\fh)\right)^2}{2 \left(\sum_{j=1}^m \|\fh_j - f_j^k\|_a\right)^2}.
\]
Define $\delta_j = \alpha_j^k - \hat\alpha_j, j\in [m]$. Since $\alpha_j^k \in \span(K_j)$ and $\hat\alpha_j \in \span(K_j)$, we have $\delta_j \in \span(K_j)$. Since we assume $\fh$ is a combination of no more than $d$ kernel classifiers, there are at most $2d$ non-zero vectors in the set $\{\delta_1, \ldots, \delta_m\}$. Using the definition of $\gamma(d,\K)$, we have
\begin{eqnarray*}
\sum_{j=1}^m \|f_j^k - \fh_j\|_a = \sqrt{N}\sum_{j=1}^m \|\delta_j\|_2 \leq \frac{\gamma(2d, \K)}{\sqrt{N}} \left\|\sum_{j=1}^m K_j (\alpha_j^k - \hat\alpha_j) \right\|_2 = \gamma(2d, \K)\|f^k - \fh\|_{\ld}.
\end{eqnarray*}
To simplify our notation, we define $\gamma = \gamma(2d, \K)$. We have
\begin{eqnarray*}
\E_N(f^k) - \E_N(f^{k+1}) & \geq & \frac{(\E_N(f^k) - \E_N(\fh))^2}{2\gamma\|f^k - \fh\|_{\ld}^2}
 \geq  \frac{(\E_N(f^k) - \E_N(\fh))^2}{4\gamma\left(\|f^k - f^*\|_{L_2}^2 + \|\fh - f^*\|_{\ld}^2\right)} \\
& \geq & \frac{(\E_N(f^k) - \E_N(\fh))^2}{8\gamma\left(\E_N(f^k) - \E_N(f^*) + \E_N(\fh) - \E_N(f^*)\right)}.
\end{eqnarray*}
The last step in the above inequality follows the fact that $f^*$ is the minimizer of the empirical loss $\E_N(f)$ and therefore
\[
\E_N(f) - \E_N(f^*)\geq \frac{1}{2N}\sum_{i=1}^N(f(\x_i) - f^*(\x_i))^2 = \frac{1}{2}\|f-f^*\|^2_{\ld}.
\]
Let $k(\mu)$ be the iteration index such that for any $k \leq k(\mu)$ we have $\E_N(f^k) - \E_N(f^*) \geq \mu (\E_N(\fh) - \E_N(f^*)) = \mu \varepsilon_*$, where $\mu \geq 1$. Then, for all $k \leq k(\mu)$, we have
\[
    \E_N(f^k) - \E_N(f^{k+1}) \geq \frac{(\mu - 1)^2}{8\gamma \mu(\mu+1)}\left(\E_N(f^k) - \E_N(f^*) \right).
\]
Define $\epsilon_k = \E_N(f^k) - \E_N(\fh)$ and $\tau = \frac{(\mu - 1)^2}{8\gamma \mu(\mu+1)}$. Then, for any $k \leq k(\mu)$, we have  $\epsilon_{k+1}  \leq \max(0, 1 - \tau)\epsilon_k $ and therefore
\[
\epsilon_{k} \leq [\max(0, 1 - \tau)]^k \epsilon_0 = [\max(0, 1-\tau)]^k\frac{\|\y\|^2_2}{2N} \le \frac{1}{2}\left[\max(0, 1 - \tau)\right]^k,
\]
leading to the result in the lemma.


\section*{Appendix C. [Proof of Proposition~\ref{prop:3}]}
We only need to prove $\widehat\gamma =\frac{\sqrt{d} }{\sqrt{1 - (d-1)\delta(\K)}\sigma^+_{\min}}$ satisfies the following inequality
\[
\sum_{j \in J} \|\a_j\|_2 \leq \widehat\gamma \left\|\sum_{j\in J} \Kh_j\a_j \right\|_2.
\]
To prove this, we let $\mathbf z= (\|\a_j\|_2,j\in J)$, 
 and proceed as follows:
\begin{eqnarray*}
&\widehat\gamma^2& \left\|\sum_{j\in J} \Kh_j\a_j \right\|^2_2\geq \widehat\gamma^2\left(\sum_{j\in J}\|\Kh_j\a_j\|_2^2 + \sum_{i\neq j, i, j\in J}\langle \Kh_i\a_i, \Kh_j\a_j\rangle\right)\\
&\geq& \widehat\gamma^2(\sigma^+_{\min})^2 \left(\sum_{j\in J}\|\a_j\|_2^2 - \delta(\K)\sum_{i\neq j, i, j\in J}\|\a_i\|_2\|\a_j\|_2\right)\\
&\geq&\widehat\gamma^2(\sigma^+_{\min})^2 \left((1-\delta(\K))\sum_{j\in J}\|\a_j\|_2^2 + \delta(\K)(2\|\mathbf z\|_2^2 - (\mathbf z^{\top}\mathbf 1)^2)\right)\\
&\geq & \widehat\gamma^2(\sigma^+_{\min})^2\left(\sum_{j\in J}\|\a_j\|^2_2\right)(1-(d-1)\delta(\K)) \geq  \widehat\gamma^2\frac{(\sigma^+_{\min})^2}{d}\left(\sum_{j\in J}\|\a_j\|_2\right)^2(1-(d-1)\delta(\K)).
\end{eqnarray*}
Plugging the values of $\widehat\gamma$, we prove the required inequality.

 \section*{Appendix D. [Proof of Lemma~\ref{lemma:gen-bound-1}]}
We first bound the concentration of regression error for fixed $r$. Using the Telagrand inequality~\citep{Koltchinskii:2011:oracle}, we have with a probability $1 - e^{-t}$
\begin{eqnarray*}
& &\hspace*{-0.3in}  \sup\limits_{\|f - g\| \leq r} |(P - P_N)(\ell \circ f - \ell \circ g)| \\
& \leq & 2\left(\mathrm{E}\left[\sup\limits_{\|f - g\| \leq r } |(P - P_N)(\ell \circ f - \ell \circ g)| \right] + \sqrt{P(\ell \circ f - \ell \circ g)^2}\sqrt{\frac{t}{N}} + |\ell\circ f - \ell \circ g|_{\infty} \frac{t}{N}\right) \\
& \leq & 2\left(\mathrm{E}\left[\sup\limits_{\|f - g\| \leq r} |(P - P_N)(\ell \circ f - \ell \circ g)| \right] +  L r\sqrt{\frac{t}{N}} + \frac{L r t}{N}\right).
\end{eqnarray*}
We now bound the expectation $\mathrm{E}\left[\sup\limits_{\|f - g\| \leq r } |(P - P_N)(\ell \circ f - \ell \circ g)| \right]$. We have
\begin{eqnarray*}
& &\hspace*{-0.3in} \mathrm{E}\left[\sup\limits_{\|f - g\| \leq r} |(P - P_N)(\ell \circ f - \ell \circ g)| \right] \\
& \leq & 2\mathrm{E}_{N,\sigma}\left[\sup\limits_{\|f - g\| \leq r} R_n(\ell \circ f - \ell \circ g)\right] \leq 4 L \mathrm{E}_{N,\sigma}\left[\sup\limits_{\|f - g\| \leq r } R_N(f - g) \right],
\end{eqnarray*}
where $R_N(f) = \frac{1}{N}\sum_{i=1}^N \sigma_i f(\x_i)$ is the Rademacher complexity measure and $\sigma_i, i=1, \ldots, N$ are Rademacher variables. The last inequality follows the contraction property of Rademacher complexity measure~\citep{Koltchinskii:2011:oracle}. To continue bounding the quantity, we first notice that
\[
\sup\limits_{\|f - g\| \leq r } R_N(f - g) \leq r\max\limits_{1 \leq j \leq m} \left[\sup\limits_{\|f_j - g_j\|_{\H_j} \leq 1} R_N(f_j - g_j) \right].
\]
This is because
\begin{eqnarray*}
\sup\limits_{\|f - g\| \leq r } R_N(f - g) 
&=&  r\sup\limits_{\sum_{j=1}^m \|f_j - g_j\|_{\H_j} \leq 1} \frac{1}{N}\sum_{i=1}^N \sum_{j=1}^m \sigma_i(f_j(\x'_i) - g_j(\x'_i)) \\
& = & r\sup\limits_{\sum_{j=1}^m \|f_j - g_j\|_{\H_j} \leq 1} \sum_{j=1}^m\frac{\|f_j - g_j\|_{\H_j}}{N}\sum_{i=1}^N  \sigma_i\frac{f_j(\x'_i) - g_j(\x'_i)}{\|f_j - g_j\|_{\H_j}} \\
& \leq&  r\max\limits_{1 \leq j \leq m} \sup_{\|f_j - g_j\|_{\H_j} \leq 1 } R_N(f_j - g_j).
\end{eqnarray*}
Using Theorem 5 from~\citep{hussain:2011:note}, we have, with a probability $1 - e^{-t}$, that
\begin{eqnarray*}
\mathrm{E}_{\sigma}\left[ \max\limits_{1 \leq j \leq m} \sup\limits_{\|f_j - g_j\|_{\H_j} \leq 1} R_N(f_j - g_j)\right] & \leq & \max\limits_{1 \leq j \leq m} \mathrm{E}_{\sigma}\left[\sup\limits_{\|f_j - g_j\|_{\H_j} \leq 1} R_N(f_j - g_j)\right] + 4\sqrt{\frac{\ln(m+1) + t}{2N}} \\
& \leq & \frac{1}{\sqrt{N}} + 4\sqrt{\frac{\ln(m+1) + t}{2N}},
\end{eqnarray*}
where the last step uses the fact $\kappa_j(\x, \x) \leq 1$ and the result from~\citep{Bartlett:2002:localrademacher}. Combining the above results and setting $t = A\ln (m+1)$, we have with a probability at least $1 - 2(m+1)^A$, for a fixed $r$,
\begin{eqnarray}
\lefteqn{\sup\limits_{\|f - g\| \leq r} |(P - P_N)(\ell \circ f - \ell \circ g)|} \nonumber \\
& \leq & 2Lr\left( \frac{4}{\sqrt{N}} + 16\sqrt{\frac{(A+1)\ln(m+1)}{2N}}+\sqrt{\frac{A\ln (m+1)}{N}} + \frac{A\ln (m+1)}{N}\right) \nonumber \\
& \leq & Lr\left(42\sqrt{\frac{A\ln(m+1)}{N}} + 2\frac{A\ln(m+1)}{N} \right). \label{eqn:a1}
\end{eqnarray}
Now, we show the bound holds uniformly for all $r \in (r_0, 2R)$. Note that $r$ cannot be larger than $2R$ because
$\sum_{i=1}^m \|f_i - g_i\|_{\H_i}  \leq 2R$. 
To this end, we consider $R_j = 2^{1-j} R, j =0, \ldots, j_0$, where $j_0 \leq \lceil \log_2[2R] - \log_2 r_0 \rceil \leq 0.5\log_2 N -1 $. Then, with probability $1 - [\log_2 N] (m+1)^{-A}$, we have (\ref{eqn:a1}) hold for all $\{R_j\}_{j=0}^{j_0}$. Using the monotonicity with respect to $r$, for any $r \geq r_0$, we have

\begin{eqnarray*}
\sup\limits_{\|f - g\| \leq r} |(P - P_N)(\ell \circ f - \ell \circ g)|  \leq  Lr\left(84\sqrt{\frac{A\ln(m+1)}{N}} + 4\frac{A\ln(m+1)}{N} \right)  \leq  88Lr\sqrt{\frac{A\ln(m+1)}{N}}.
\end{eqnarray*}

We complete the proof by using the relation $\log_2 N < m + 1$ and $N \geq A\ln(m+1)$.



\end{document}